\documentclass{article} 
\usepackage{nips13submit_e,times}
\usepackage{hyperref}
\usepackage{url}

\usepackage{graphics} 
\usepackage{epsfig} 
\usepackage{subfigure}
\usepackage{booktabs}
\usepackage{float}
\usepackage{verbatim} 
\usepackage[margin=1in]{geometry}
\usepackage{amssymb,amsfonts,bbm,stmaryrd}
\usepackage{algorithm,algorithmicx,listings}  
\usepackage[noend]{algpseudocode}	
\usepackage[mathscr]{euscript}
\usepackage{times}                      
\usepackage{color}                       
\usepackage{amsmath}
\usepackage{graphicx,psfrag}                   
\usepackage{makeidx}                            
\usepackage{listings}                          
\usepackage{dsfont}

\graphicspath{{figs/}}

\nipsfinalcopy
\newtheorem{theorem}{Theorem}[section]
\newtheorem{lemma}[theorem]{Lemma}
\newtheorem{proof}[theorem]{Proof}
\newtheorem{definitions}{Definition}

\title{\LARGE \bf
Learning and Optimization with Submodular Functions \\
}
\author{Bharath Sankaran, Marjan Ghazvininejad, Xinran He, David Kale, Liron Cohen\\
Department of Computer Science\\
University of Southern California\\
Los Angeles, CA 90034\\
\texttt{\{bsankara,mghazvin,xinranhe,dkale,lironcoh\}@usc.edu}
}

\begin{document}
\maketitle
\section{Motivation}
In many naturally occurring optimization problems one needs to ensure that the definition of the optimization problem lends itself to solutions that are tractable to compute. In cases where exact solutions cannot be computed tractably, it is beneficial to have strong guarantees on the tractable approximate solutions. In order operate under these criterion most optimization problems are cast under the umbrella of convexity or submodularity. In this report we will study design and optimization over a common class of functions called submodular functions.

Set functions, and specifically submodular set functions, characterize a wide variety of naturally occurring optimization problems, and the property of submodularity of set functions has deep theoretical consequences with wide ranging applications. Informally, the property of submodularity of set functions concerns the intuitive \textit{principle of diminishing returns}. This property states that adding an element to a smaller set has more value than adding it to a larger set. Common examples of submodular monotone functions are entropies, concave functions of cardinality, and matroid rank functions; non-monotone examples include graph cuts, network flows, and mutual information.

In this paper we will review the formal definition of submodularity; the optimization of submodular functions, both maximization and minimization; and finally discuss some applications in relation to learning and reasoning using submodular functions.

\section{What is Submodularity: Formal Definition}
We define submodularity as the property of set functions $f:2^V \rightarrow \mathds{R}$, which assign to each subset $S \subseteq V$ a value $f(S)$. Here $V$ is a finite set called the {\bf ground set}. We also assume that $f(\emptyset) = 0$.

\begin{definitions}
A set function $f:2^V \rightarrow \mathds{R}$ is called \textbf{submodular} if it satisfies
\[
 f(X) + f(Y) \geq f(X \cup U) + f(X \cap Y) \text{  } \forall \text{  } X,Y \subseteq V
\]
\end{definitions}

The function $f$ lends itself to different forms in different application domains. In a machine learning context $f$ could be a function that evaluates information of a given set, i.e entropy. Using this notion, we can easily introduce the property of diminishing returns by using an equivalent definition for submodularity.

\begin{definitions}
A set function $f:2^V \rightarrow \mathds{R}$ is called submodular if it satisfies
\[
 X \rightarrow f(X\cup {k}) - f(X) \text{ is non-increasing }
 \]
 \[
 f(X \cup {k}) - f(X) \geq f(Y \cup {k}) + f(Y) \text{  } \forall \text{  } X\subset Y \text{  } \forall k \notin X
\]
\end{definitions}

Finally, a set function $f$ is called {\bf supermodular} if $-f$ is submodular, and if $f$ is both sub- and supermodular then the function is called a {\bf modular function}.

\subsection{Notation I}
\label{sec:notation}
In this section we will introduce some notation that we will consistently maintain through the course of this document, unless specified. The ground set over which the submodular functions are defined will be denoted by $V$ with cardinality $n$. For a vector $x \in \mathds{R}^V$ and a subset $Y \subseteq V$ we define $x(Y) = \underset{u \in Y}{\operatorname{\sum }}\text{ } x(u)$. We can naturally extend this definition to capture the positve and negative parts of the vector $x$ as $x^+\in \mathds{R}^V$ and $x^-\in \mathds{R}^V$, where $x^+(u) = max\{x(u),0\}$ and $x^-(u) = min\{x(u),0\}$. For a submodular function $f$ we define a polyhedral convex set $P(f)$ called the {\bf submodular polyhedron}:\\
\[
 P(f) = \{ x\in \mathds{R}^V \mid x(S) \leq f(S) \text{ }\forall S\subseteq V \}
\]

The face of $P(f)$ for which $x(V) = f(V)$ which defines the {\bf base polyhedron}:\\
\[
 B(f) = \{ x \in P(f) \mid x(V) = f(V)\}
\]

The elements of $B(f)$ are bases of the set $V$ or the polyhedron $P(V)$.

\subsection{Properties of Submodular Functions}
The basic properties of submodular functions are enumerated below. These properties will help us recast many of our optimization objectives as submodular optimization problems.\\

\begin{itemize}
 \item {\bf \lemma (Closedness Properties)}: Submodular functions are closed under nonnegative linear combinations, i.e if $\{f_1,f_2,...,f_k\}$ are submodular then the function $g(X) = \sum\limits_{i=1}^k\text{ } \alpha_if_i(X) \text{ is submodular }\forall \alpha_i \geq 0$.\\
 
 {\it Corollary 1.1}: The sum of a modular and submodular function is a submodular function.\\
 
 {\it Corollary 1.2: (Restriction / marginalization)}: if $Y\subset V$, then $X\rightarrow f(X \cap Y)$ is submodular on $V$ and $Y$.\\
 
 {\it Corollary 1.3: (Contraction / conditioning)}: If $X \subseteq Y$ and $f$ is submodular, then $g(X) = f(Y\setminus X)$ is submodular. Equivalently if $Y\subset V$, then $X\rightarrow f(X \cup Y) - f(Y)$ is submodular on $V$ and $V\setminus Y$\\
 
 \item {\bf \lemma (Partial Minimization)}: Monotone submodular functions remain submodular under truncation, i.e if $f(X)$ is submodular then $g(X) := \min\{f(X),c\}$ for any constant $c$ is submodular.\\
 
 {\bf Note:} This property is not necessarily preserved for max or min for two submodular functions.\\
 
 \item{\bf \lemma (Cardinality Based Functions)} If $f(X)$ is a submodular function, then $g(X) = \phi(f(X))$ is also submodular if $\phi()$ is a concave function.\\
 
 \item{\bf \lemma (Lov$\acute{\bf a}$sz Extension)} A function $f$ is submodular function, iff its Lov$\acute{a}$sz Extension $\hat{f}$ is convex, where
 \[
  \hat{f}(c) = \max\{c^Tx \mid x(U) \leq f(U) \text{ } \forall \text{ } U \subseteq V \text{ and } c\in[0,1]^n\}
 \]

\end{itemize}

\section{Submodular Optimization}
Submodular functions have many interesting connections with convex and concave functions as demonstrated by {\bf Lemma 3}. Just as minimization of convex functions can be done efficiently, unconstrained submodular minimization is also possible in strongly polynomial time. Submodular function maximization in contrast is a NP hard combinatorial optimization problem, but approximate solutions can be found with guarantees. In fact a simple greedy solution method obtains a $(1 - 1/e)$ approximation, given that we are maximizing a non-decreasing submodular function under matroid constraints.

\subsection{Submodular Function Minimization}
Submodular function minimization can be divided into two categories, exact and approximate algorithms. Exact algorithms obtain global minimizers for a problem whereas approximate algorithms only achieve an approximate solution i.e for a set $X$ the solution $f(X) - \underset{Y \subset V}{\operatorname{min }}\text{ } f(Y) \leq \epsilon $, where $\epsilon$ is as small as possible. If $\epsilon$ is less than the minimum absolute difference between non equal values of $f$, the solution computed corresponds to the exact solution.

An important practical aspect of submodular function minimization is that most algorithms come with online approximation guarantees {\it due to a duality relationship}, which we will detail in the following subsections.\\

\subsubsection{\bf Submodular Function Minimizers}
For the lemmas stated below, we consider $f$ to be a submodular function where $\{f:2^V \rightarrow \mathds{R} \mid f(\emptyset) = 0\}$.\\

\begin{itemize}
 \item {\bf \lemma (Lattice of minimizers for submodular functions)}: The set of minimizers of $f$ is a lattice, i.e if $X$ and $Y$ are minimizers then $X\cup Y$ and $X \cap Y$ are also minimizers. This is evident from {\bf Definition 1} and {\bf Lemma 1.1}\\

 \item {\bf \lemma (Diminishing return property of minimizers of submodular functions)}: The set $X \subset V$ is a minimizer of $f$ on $2^V$ iff $X$ is a minimizer of $2^X \rightarrow \mathds{R}$ defined as $Y \subset X \rightarrow f(Y)$ and if $\emptyset$ is a minimizer of the function from $2^{V\setminus X} \rightarrow \mathds{R}$  then it is defined as $Y \subset V\setminus X \rightarrow f(Y \cup X) - f(X)$ . This can be easily shown from {\bf Definition 1}.\\
 
  {\it Corollary 6.1 : (Norm Characterization)}: Suppose $\hat{x}$ is a minimizer of 
 \[
  \underset{x}{\operatorname{min }}\text{ } \| x\|_2^2 \text{ subject to } x\in B(f)
 \]
 Then a minimizer $A$ for $f$ can be obtained as follows:
 \[
   A = \{ u \in V \mid \hat{x}(u) \leq 0\}
 \]

  \item{\bf \lemma (Dual of minimization of submodular functions)}: \[ \underset{X \subset V}{\operatorname{min }}\text{ } f(X) =  \underset{x \in B(f) }{\operatorname{max }}\text{ } x^-(V) = f(V) - \underset{x \in B(f) }{\operatorname{min }}\text{ } \|x\| \]
  As mentioned in Section \ref{sec:notation} $(x^-)_k = \min\{x_k,0\} \text{ } \forall k \in V$. If $X \subset V$ and $x \in B(f)$, we have $f(X) \geq x^-(V)$ with equality iff $\{ x < 0 \} \subset X \subset \{ x\leq 0\}$ and $x(X) = f(X)$. \\
  \[
   \underset{X \subset V}{\operatorname{min }}\text{ } f(X) =  \underset{x \in P(f), x\leq 0 }{\operatorname{max }}\text{ } x(V)
  \]
  Again if $X\subset V$ and $x\in P(f) \mid x \leq 0$, then $f(X) \geq x(V)$ iff $\{x \leq 0 \} \subset X$ and $s(X) = f(x)$.\\

\end{itemize}
\subsubsection{Minimum Norm Point Algorithm}
As an example of Submodular function minimization we present the minimum norm point algorithm. A non combinatorial approach proposed by Fujishige \cite{Fujishige} is based on the norm characterization of the minima of $f$ shown in {\bf Lemma 2.2}. Fujishige uses Wolfe's algorithm \cite{Wolfe} which was developed to minimize the $L_2$ norm of a vector in a convex hull of a finite set of points $P \in \mathds{R}^n$. This method maintains the vector $x$ as a convex combination of points $S$ and iterates over the following steps:
\begin{enumerate}
\item A new point from $P$ with a norm with respect to $x$ is added to to set $S$.
\item A point with the minimum norm $x$ is computed in the affine hull of $S$.
\item The minimum norm point $x$ is projected onto the convex hull of $S$. 
\end{enumerate}

In the case of submodular functions, one needs to search through the set of all bases $P$ which is exponential in size. This issue is circumvented by using Edmonds Greedy Algorithm \cite{Edmond}.

\begin{algorithm}[htb]
\caption{Minimum Norm Point Algorithm}
\label{alg:min_norm_point}
\begin{algorithmic}[1]
\footnotesize
\State {\bf Initialization}: x $\leftarrow$ extreme base generated using arbitrary ordering, $S \leftarrow \{x\}$
\Loop
  \State {\bf Selection of new base using Edmonds Greedy Algorithm}: $y' \leftarrow  \underset{y}{\operatorname{argmin }}\text{ } x^Ty \text{ } \forall y\in B(f)$
\If{$x^Ty' = x^Tx$}
\State {\bf return} x
\Else
\State $S \leftarrow S \cup \{y'\}$
\EndIf
\State {\bf Minimization over affine hull of $S$}: $z \leftarrow  \underset{y}{\operatorname{argmin }}\text{ } {\|y\|_2^2 \text{ where } y\in S}$
\State {\bf Projection on to convex hull of $S$}: 
\While{ $z \notin \mathrm{relint}(\mathrm{conv}(S))$}
               \State $z \leftarrow \text{ intersection of } [z,x] \text{ and } S$
		\State $S \leftarrow \text{ face of $S$ intersected by } [z,x]$
\EndWhile
\State $x\leftarrow Z$
\EndLoop
\end{algorithmic}
\end{algorithm}

\subsection{Submodular Function Maximization}
Problems for the form $\underset{X\subset V}{\operatorname{max }}\text{ } f(X)$ for any submodular function $f$ occurs in various applications. Problems of these kind are known to be NP-Hard. Feige and Mirrokni \cite{Feige07maximizingnon-monotone} showed that maximizing for non-negative submodular functions a random subset achieves at least 1/$4^{\text{th}}$ the optimal value and local search techniques achieve at least a 1/2. Though these problems are NP-Hard, a $(1 - 1/e)$ approximation can be obtained when maximizing a non-decreasing submodular function under matroid constraints.
The solution to the arbritary matroid constraint was shown more recently by Vondrak et al \cite{Vondrak}. The initial result of an $(1 - 1/e)$ approximation was shown by Nemhauser  \cite{Nemhauser} in the 70s, but the result was only applicable to uniform matroid (cardinality) constraints. The solution to the uniform matroid contraint consists of a simple greedy algorithm that has implication in online learning and adaptive submodularity.
\begin{itemize}
 \item {\bf \lemma Local minima for submodular function minimization:} \\ Given a submodular function $\{f:2^V \rightarrow \mathds{R} \mid f(\emptyset) = 0\}$ and  $X \subset V$ such that $\forall k \in X$, $\text{ }f(X\setminus \{k\}) \leq f(X)$  and $\forall k \in V\setminus X, \text{ } f(X\cup \{k\}) \leq f(X)$ holds true.
 \[ \text{Then, } \forall Y\subset X \text{ and } \forall Y\supset X \text{ , }f(Y) \leq f(X) \] 
\end{itemize}
\subsubsection{Greedy Algorithm for Monotone Submodular Function Maximization with Uniform Matroid Constraints}
Maximization for arbitary constraints can be achieved using Vondrak's algorithm. In this document we'll focus on the greedy algorithm as it has implications in the online learning domain. {\bf Note:} Maximization can also be formulated using the base polyhedron given we have $f$ and its lov$\acute{a}$sz extension $\hat{f}$. In this case maximization is equivalent to finding the maximum l1-norm point in the base polyhedron. See \cite{Bach} for more details.
For monotone submodular maximization subject to uniform matroid constraints, we need to find a set $X^*\subseteq V$ such that
\[
 X^* = \underset{\|X\| \leq n}{\operatorname{argmax }}\text{ }f(X)
\]

where $n$ is the cardinality (uniform matroid) constraint. Though this problem is NP-hard we can get an approximate solution with an approximation of $(1-1/e)$ of the optimal solution.  The algorithm for obtaining this solution is shown in Algorithm \ref{alg:greedy}.
\begin{algorithm}[htb]
\caption{Greedy Algorithm}
\label{alg:greedy}
\begin{algorithmic}[1]
\footnotesize
\State {\bf Initialization}: Start with $X = \emptyset$
\For{ i = 1 to n }
  \State $y':= \underset{y}{\operatorname{argmax }}\text{ }f(X\cup {y}) $
  \State $X := X \cup {y'}$
\EndFor
\end{algorithmic}
\end{algorithm}

\subsection{Adaptive Submodularity}
The process of adaptively making decisions with uncertain outcomes is fundamental to many problems with partial observability. In such situations the decision maker needs to make a sequence of decisions by accounting for past observations and adapting accordingly. It has been shown by Golovin and Krause \cite{Golovin} that if a problem is adaptively submodular, then an adaptive greedy algorithm is guaranteed to obtain near optimal solutions.
For clarity in the discussion of the \textbf{Applications} section, we introduce the notion of adaptive submodularity and adaptive monotonicity in this section.

\subsubsection{Preliminaries and Notation II}
Let $V$ be the ground set. Assuming each item of the set $x\in V$ can take a number of states from a set of possible states $O$, we represent item states as $\phi:V \rightarrow O$ which is a function that gives the realization of the states of all items in the ground set. Hence $\phi(x)$ is the state of $x$ under the realization $\phi$. Now consider a random realization characterized by the random variable $\Phi$. Then we can assume a prior probability distribution over realizations as $p(\phi) := \mathbb{P}[\Phi = \phi]$. 
In cases where we observe only one realization $\Phi(x)$ at a time, as we pick an item $x \in V$ at a time; we can represent our observations so far with a partial realization $\chi$, i.e a function of a subset of $V$ and its observed states. Hence $\chi \subseteq V \times O$ is $\{(x,o):\chi(x) = o\}$. Here we denote the domain of $\chi$, i.e the set of items observed in $\chi$ as $dom(\chi) = \{ x: \exists o.(x,o) \in \chi \}$. When a partial realization $\chi$ is equal everywhere with $\phi$ in the $dom(\chi)$, they are {\it consistent} $\phi \sim \chi$.
This implies that all the items observed with specific states in $\chi$ have also been observed with the same states in $\phi$. Now we extend this notion to subsets by saying if $\chi$ and $\chi'$ are consistent with $\phi$ and $dom(\chi)\subseteq dom(\chi')$, then $\chi$ is a subrealization of $\chi'$. In a Partially Observable Markov Decision Problem (POMDP) sense, partial realization encompasses the POMDP belief states. These determine our posterior belief given the effect of all our actions and observations.
\[
 p(\phi\mid\chi) := \mathbb{P}[\Phi=\phi | \Phi \sim \chi]
\]

{\bf \definitions (Conditional Expected Marginal Benefit):} 
{\it Given a partial realization $\chi$ and a item $x$, the conditional expected marginal belief of $x$ conditioned on having already observed $\chi$ is denoted by $\delta(x\mid\chi)$
\[
 \delta(x\mid\chi) = \mathbb{E}[f(dom(\chi)\cup\{x\},\Phi) - f(dom(\chi),\Phi) \text{ } \mid \text{ } \Phi \sim \chi]
\]
}
{\bf \definitions (Adaptive Monotonicity):}
{\it A function $f:2^V \times O^V \rightarrow \Re_+$ is adaptive monotone with respect to the distribution $p(\phi)$ if the conditional expected marginal benefit of any item $x$ is non-negative, i.e $\forall \chi$ with $\mathbb{P}[\Phi \sim \chi] \geq 0$ and all $x\in V$ }
\[
 \delta(x\mid\chi) \geq 0
\]

{\bf \definitions (Adaptive Submodularity):}
{\it A function $f:2^V \times O^V \rightarrow \Re_+$ is adaptive submodular with respect to the distribution $p(\phi)$ if the conditional expected marginal benefit of any fixed item does not increase as more items are selected and their states are observed, i.e if $f$ is adaptively submodular w.r.t to $p(\phi)$ if $\forall \text{ } \chi$ and $\chi'$ where $\chi$ is a subrealization of $\chi'$ and all $x\in V \setminus dom(\chi')$ , the following condition holds true:}
\[
 \delta(x\mid\chi) \geq \delta(x\mid\chi')
\]
Given these definitions we can now use the greedy algorithm defined in Algorithm \ref{alg:greedy_adap} to give an $\alpha$ approximation to the best greedy solution for online maximization problems of adaptively montone submodular functions. This means we find an $x'$ such that
\[
 \delta(x'\mid\chi) \geq \frac{1}{\alpha}\delta(x\mid\chi)
\]

The budget for these maximization problems OR the number of rounds we'd like to maximize is similar to the cardinality constraint of submodular problems.

\begin{algorithm}[htb]
\caption{$\alpha$-Approximate Greedy Adaptive Algorithm}
\label{alg:greedy_adap}
\begin{algorithmic}[1]
\footnotesize
\State {\bf Input}: Budget $n$, ground set $V$, $p(\phi)$ and function $f$ 
\State {\bf Output}: $X \subset V$ where $\|X\| = n$
\State {\bf Initialize:} $X \leftarrow \emptyset$ and $\chi \leftarrow \emptyset$
\For{ i = 1 to n }
  \State $\forall x \in V\setminus X \text{; Evaluate } \delta(x\mid\chi) = \mathbb{E}[f(dom(\chi)\cup\{x\},\Phi) - f(dom(\chi),\Phi) \text{ } \mid \text{ } \Phi \sim \chi]$
  \State $x^* = \underset{x}{\operatorname{argmax }}\text{ } \delta(x\mid\chi)$
  \State $X \leftarrow X \cup {x^*}$
  \State $\text{Observe :} \Phi(x^*) \text{; Update: } \chi \leftarrow \chi \cup {x^*,\Phi(x^*)}$
\EndFor
\end{algorithmic}
\end{algorithm}

\section{Applications}
\subsection{Feature Selection}
In machine learning and statistics, \textit{feature selection} is one of the most important concepts. The aim of this process is to select a subset of relevant features for use in model construction and parameter fitting.  In real world problems, we often begin learning with a large number of candidate features that may be redundant, irrelevant, or noisy. Such features needlessly increase the complexity of our models and may lead to overfitting and poor generalization to previously unseen data. By pruning out redundant or irrelevant features, we can gain:
\begin{itemize}
\item improved model interpretability
\item increased computational efficiency, particularly during parameter fitting and prediction
\item enhanced generalization of the model by reducing overfitting
\end{itemize} 

In feature selection, we search among features and choose the ones that are, in a broad sense, most informative or useful. This definition can be interpreted as an optimization problem for choosing a subset of features which maximize the mutual information between features and labeling function.

Hence, if $V$ indicates the set of all features and a binary vector $S$ indicating the chosen feature set, and $x_S$ the real valued vector of feature values for features in $S$, then assuming that $||S||_1 \leq b$, we can write the problem as:
\begin{equation*}
max_s I(y;x_s)
\end{equation*} 
where $y$ is the labeling function.

\subsubsection{Submodularity}

We will now show that this is submodular. Suppose $A \subset B \subset S$ and $m \not \in B$. We can show that

\begin{eqnarray}
I(y;x_A \cup x_m) - I(y;x_A) &\geq & I(y;x_B \cup x_m) - I(y;x_B) \nonumber \\
\Leftrightarrow H(y|x_A)-H(y|x_A,x_m) &\geq & H(y|x_B)-H(y|x_B,x_m) \label{f1}
\end{eqnarray}
  
 we can write
\begin{eqnarray} 
	&&H(y|x_A)-H(y|x_A,x_m) \nonumber \\
	&=&H(y|x_A)+ H(x_m|x_A)-H(y,x_m|x_A)\nonumber \\
	&=&H(y|x_A)+ H(x_m|x_A) -H(y|x_A) - H(x_m|x_A,y)\nonumber \\
	&=& H(x_m|x_A)- H(x_m|x_A,y) \label{f2}
\end{eqnarray} 

By substituting into equation \ref{f1}, we can see that there are cases where this problem is not submodular. Here we give the necessary conditions for submodularity:

\begin{itemize}
\item {\bf \lemma} If $x_i$'s are all conditionally independent given y, then the function is submodular \cite{krausefeature}.
\end{itemize}

This constraint is met in many practical machine learning problems. If the $x_i$'s are all conditionally independent given $y$, then equation \ref{f2} can be written as

\begin{equation*}
H(y|x_A)-H(y|x_A,x_m)=H(x_m|x_A)-H(x_m|x_A,y)
\end{equation*}

and if we substitute this in equation \ref{f1}, then it follows that 

\begin{eqnarray*}
I(y;x_A \cup x_m) - I(y;x_A) &\geq & I(y;x_B \cup x_m) - I(y;x_B) \nonumber \\
\Leftrightarrow H(x_m|x_A) &\geq & H(x_m|x_B)
\end{eqnarray*} 

Hence, the problem of feature selection can be written as a maximization of a submodular function.\cite{jie}

\subsection{MAP Inference}
In this section we will specifically look at the problem of Maximum a posteriori inference on graphs. To analyze the algorithms in greater detail, we would like to introduce a few preliminary notions, including the concept of Polymatroids. 
\subsubsection{Polymatroids}
The notion of submodularity was first studied in the context of matroids. A set system $(V,\mathcal{F})$ is defined by a ground set $V$ and a family of subsets $\mathcal{F} \subseteq 2^V$. Such a system is a matroid if
\begin{itemize}

\item $\emptyset \in \mathcal{F}$
\item $\text{ if } X \subseteq Y \in \mathcal{F} \text{ then } X\in \mathcal{F}$
\item $\text{ if } X,Y \in \mathcal{F} \text{ and } \|X\| > \|Y\| \text{ } \exists e\in X\setminus Y \text{ such that } Y+e\in\mathcal{F}$
\end{itemize}

Now we define a function $\rho$ called a rank function, which assigns a natural number to each subset of $V$. This rank function is analagous to rank functions of matrices, in fact a matroid which is the set of linearly independent columns of a matrix $A$ is called a metric matroid. We define our matroid {\it rank function} $\rho$ as follows
\[
\rho(X) = max\{\|F\| \text{ }\mid \text{ } F\in\mathcal{F}, \text{ }F\subseteq X\}
\]
If $\rho$ is a rank function of a matroid $(V,\mathcal{F})$ then the following properties hold:
\begin{itemize}
\item $\rho(X) \leq \|X\| \text{ }\forall X\subseteq Y$
\item $\rho \text{ is non decreasing: if } X \subseteq Y \subseteq V \text{ then } \rho(X) \leq \rho(Y)$
\item $\rho \text{ is submodular }$
\end{itemize}

If a set function $\rho$ satisfies the above properties for a ground set $V$ then the resulting structure $(V,\rho)$ is called a {\it polymatroid}. Similarly if $(V,\rho)$ is a polymatroid, then the family of subsets
\[
\mathcal{F} = \{ F\subseteq V \text{ }\mid \text{ } \rho(F) = \|F\|\}
\]
defines a matroid $(V,\mathcal{F})$.

\subsubsection{Cuts in Graphs, Energy Minimization and MAP Inference}
Consider a directed graph $G = (V,A,W)$ with positive edge weights $w:A\rightarrow\mathds{R}^+$. We can define a {\bf positive directed cut} for a given set of vertices $S\subseteq V$ as the set of edges starting in $S$ and ending in $V\setminus S$ $:\delta^+(S) = \{(i,j)\in A\text{ }\mid \text{ }i \in S, j\in V\setminus S\}$, Similarly a negative directed cut is $\delta^-(S) = \{(i,j)\in A\text{ }\mid \text{ }i \in S, j\in V\setminus S\}$. Finally the cut $\delta(S) = \delta^+\cup\delta^-$, this for an undirected graph would be the set of edges with exactly one end in $S$. Hence we can define the weight of a cut as
\[
 f^+ = \underset{e\in\delta^+(S)}{\operatorname{\sum}} w(e), \text{ }  f^- = \underset{e\in\delta^-(S)}{\operatorname{\sum}} w(e), \text{ } f = \underset{e\in\delta^+(S)}{\operatorname{\sum}} w(e)
\]

Given these cut functions one can note that these cut functions are submodular.

\begin{itemize}
\item {\bf \lemma} {\it The cut functions $f^+$, $f^-$ and $f$ are submodular}\\
{\it Proof :} For the function $f$, suppose $X,Y\subset V$ then,
\[
f(X) + f(Y) - f(X\cup Y) - f(X\cap Y) = \underset{i\in\{X\setminus Y\}, j\in\{Y\setminus X\}}{\operatorname{\sum}} w(i,j) + \underset{i\in\{X\setminus Y\}, j\in\{Y\setminus X\}}{\operatorname{\sum}} w(j,i)
\]
from the non-negativity of edge weights we can quickly conclude the above function is submodular. Similarly submodularity can be proved for $f^+$ and $f^-$.
\end{itemize}
 Now in order to formalize our notion of Maximum a posteriori estimation as a submodular function minimization problem, we introduce the following notation. Consider the function $E:\{0,1\}^n\rightarrow\mathds{R}$ defined over binary variables $X=\{x_1,...,x_n\}$. Such functions are called \textbf{regular functions} \cite{Kolmogorov04whatenergy}. We can define an equivalent set function $\hat{E}$
\[
\hat{E}(S) = E(x) \text{ where } x_i = 1 \text{ if and only if } i \in S
\]
We define the class $\mathcal{F}^2$ to be functions that can be written as a sum of functions of up to two binary variables at a time.
\[
E(x_1,....x_n) = \underset{i}{\operatorname{\sum}} E^i(x_i) + \underset{i<j}{\operatorname{\sum}} E^{i,j}(x_i,x_j)
\]
The regularity of the binary function from $\mathcal{F}^2$ translates to submodularity of the equivalent set function. 

Given an input set of nodes $\mathcal{P}$ in a graph $\mathcal{G}$ and a set of labels $\mathcal{L}$, the labeling $l$ (which is a mapping from $\mathcal{P}$ to $\mathcal{L}$) can be deduced by minimizing some energy function. In graph based energy minimization problems in computer vision and machine learning, the standard form of the energy function used is as follows
\[
E(l) = \underset{p\in\mathcal{P}}{\operatorname{\sum}} D_p(l_p) + \underset{p,q\in\mathcal{N}}{\operatorname{\sum}} V_{p,q}(l_p,l_q)
\]

where $\mathcal{N} \subset \mathcal{P}\times\mathcal{P}$ is a neighbourhood set of nodes. $D_p$ is a cost function derived from assigning label $l_p$ to node $p$. $V_{p,q}$ is the cost of assigning labels $l_p,l_q$ to adjacent nodes $p,q$. If V is a non-convex function of $\|l_p - l_q\|$ which accounts for border labeling, the energy function $E(l)$ is called a discontinuity preserving energy function. This label assignment problem is similar to the graph cut problem, as the labeling function is submodular in the context of the regular functions defined earlier. Minimizing this energy function E is equivalent to finding the minimum cut of the graph $\mathcal{G}$. However one should note that the solution depends on the exact form of the function $V$ and it cannot be convex as it leads to oversmoothing of borders. 
In the case when $V(l_p,l_q)=T[l_p\neq\l_q]$, where T is the indicator function. This smoothness term is called the Potts Model. The solution shown above can readilyb e extended to more than two labels, or beyond the binary problem. We use the binary problem to motivate the result shown above. This result is widely used in computer vision in the domains of image segmentation, stereo correspondence and multi-camera image reconstruction. Another widely used application of this approach is to find the Maximum a posteriori estimate of a Markov Random Field \cite{MRFKohli}.

Consider a set of random variables $X = \{X_1,....,X_n\}$ defined on a set $S$ such that the variable $X_i$ can take the value $x_i$ from the set $\mathcal{L} = \{l_1,...l_n\}$. Then $X$ can be defined as a Markov Random field with respect to the neighbourhood set $N = \{N_i \text{ } \mid \text{ } i\in S\}$ iff, the positivity property $P(x) > 0$ and the Markovian property $P(x_i\mid x_{S\setminus {i}}) = P(x_i \mid x_{N_i}) \text{ } \forall i \in S$. Here $P(x) = P(X = x)$, $P(x_i) = P(X_i = x_i)$ and finally $P(X_1 = x_1,...,X_n = x_n)  = (X = x) \text{ where } x = \{x_i \mid i\in S\}$ is a realization of the field. 

Given these definitions the MAP estimate of the MRF can be formulated as an energy minimization problem, where energy corresponding to a realization of x (configuration of the field) is given by the negative log likelihood of the joint posterior probability of the MRF
\[
\phi(x) = -logP(x\mid D)
\]
Hence the corresponding energy function for the Potts model becomes 
\[
E(x) =  \underset{i\in S}{\operatorname{\sum}} \left(\phi(D|x_i) + \underset{j\in \mathcal{N}_i}{\operatorname{\sum}}\psi(x_i,x_j) \right)
\]

where 
\[
\phi(D|x_i)  = -logP(i\in S)
\]
 and 
 \[
 \psi(x_i,x_j) = \left\{ 
  \begin{array}{l l}
    K_{ij} & \quad \text{if $x_i \neq x_j$}\\
    0 & \quad \text{if $x_i = x_j$}
  \end{array} \right.
 \]
 Here $K_{ij}$ is some penalty cost which makes $\psi(x_i,x_j)$ non convex.

 Finally we can conclude that {\it the energy minimization problem solved by min-cut, max flow which yields the minimum energy solution is equivalent to finding the maximium a posteriori solution of a Markov Random Field}.

\subsection{Active Learning}
\subsubsection{Supervised learning theory}

In classic supervised machine learning, the learning algorithm (or \textit{learner}) is given the task of finding a response function $f: \mathcal{X} \mapsto \mathcal{Y}$ that predicts as accurately as possible the output \textit{response} $Y \in \mathcal{Y}$ for a given input observation	 $X \in \mathcal{X}$ \cite{Mohri:2012}. Responses take a variety of forms. In classification, this may be a label from a discrete set of choices $\mathcal{Y} = \{ 1, 2, \dots\}$, while in regression it may be continuous. One of the most common tasks is binary classification, in which $\mathcal{Y} = \pm1$. We have some unknown underlying distribution $\mathcal{D}$ over the space of observations and responses $\mathcal{X} \times \mathcal{Y}$, so that observation-response pairs are sampled according to $(X, Y) \thicksim \mathcal{D}$. The learner chooses from candidate functions or \textit{hypotheses} in a hypothesis space $\mathcal{H}$ with the goal of minimizing the expected error or \textit{risk} $\epsilon_\mathcal{D}(h) = \mathbf{E}_{(X,Y)\thicksim \mathcal{D}}[\mathrm{err}(h(X), Y)]$. In other words, the learner's goal is to find $h^\ast$ that minimizes the risk: $h^\ast = \arg\max_{h \in \mathcal{H}} \epsilon_{\mathcal{D}}(h)$. For standard classification tasks, the error function is simply the indicator function of a mistake $\mathds{1}\{h(X) \not= Y\}$, and so the risk is simply the probability of a mistake $\epsilon(h) = \mathbf{E}_{(X,Y)\thicksim \mathcal{D}}[\mathds{1}\{h(X) \not= Y\}] = \mathbf{Pr}\{h(X) \not= Y\}$. For continuous response functions and multiclass classification where order matters, there are a wide choice of more complex error functions.

Of course, in practice $\mathcal{D}$ is unknown and so it is impossible to directly minimize the risk. Instead, the learner is provided with ``supervision'' in the form of a finite sample of observation-response pairs, i.e., a labeled \textit{training} data set $\mathcal{S} = \{ X_i, Y_i \}_{i=1, \dots, n}$ where $|\mathcal{S}| = n$. The learner can then approximate $\mathcal{D}$ using $\mathcal{S}$ and minimize the empirical error over $\mathcal{S}$:
\[
\hat{\epsilon}_{\mathcal{S}}(h) = \mathbf{E}_{(X,Y) \in \mathcal{S}}[\mathrm{err}(h(X), Y)] = \frac{1}{n} \sum_{i=1}^n \mathrm{err}(h(X_i), Y_i)
\]

\noindent Note that this definition of empirical risk assumes that samples $(X, Y)$ are identically independently distributed (IID), a fairly common assumption in supervised machine learning. In the \textbf{empirical risk minimization} (ERM) paradigm, the learner assumes that the sample $\mathcal{S}$ is sufficiently representative of $\mathcal{D}$ such that choosing $\hat{h} = \arg\max_{h \in \mathcal{H}} \hat{\epsilon}_{\mathcal{S}}$ will yield a hypothesis $\hat{h}$ that will also have a relatively low risk $\epsilon_{\mathcal{D}}(\hat{h})$ \cite{Vapnik:2000}. A well known theoretical result for classification that comes from Vapnik tells us if we want to learn a ``good'' classifier from a hypothesis class $\mathcal{H}$, then we need roughly $|\mathcal{S}| = \widetilde{O}\left(d/\varepsilon^2 \log (1/\delta)\right)$ points in our training sample \cite{Vapnik:1999}. Here $\varepsilon$ is the maximum deviation that we will tolerate between the true risks of $\hat{h}$ and optimal $h^\ast$ and $\delta$ is the probability with which we are willing to let this happen (i.e., we want $|\epsilon(\hat{h}) - \epsilon(h^\ast)| \leq \varepsilon$ to hold with probability $1-\delta$). Informally, $d$ represents the ``size'' of our hypothesis class; formally, it is the \textbf{VC dimension}. A useful rule of thumb is that for most useful hypothesis classes, the VC dimension scales linearly with the number of parameters and so the number of training samples needed scales linearly with ``complexity'' of the model.

It is important to distinguish two cases of supervised learning, based on realizability. When the problem is \textbf{realizable}, then there exists some hypothesis $h \in \mathcal{H}$ that can perfectly predict the response for every point (i.e., $\mathrm{err}(h^\ast) = 0$); in binary classification, this corresponds to the problem being ``separable'' by a hypothesis in $\mathcal{H}$. When $\mathrm{err}(h) > 0$, the problem is not realizable \cite{dasgupta2011}. The presence of \textit{label noise}, where the same point may receive different responses, further complicates this picture. If a training sample $\mathcal{S}$ contains noisy labels (perhaps due to error), this may mislead the ERM. If the true data distribution allows points to have different labels (i.e., our true labeling function is stochastic), then at best we may only be able to model $P(Y|X)$, rather than make perfect predictions.

\subsubsection{Selective sampling as a submodular problem}

Imagine the following problem, which we will call the \textbf{selective sampling on a budget} problem: given a large, fully labeled finite sample $\mathcal{S}$, we will ``purchase'' a subset $\mathcal{L} \subseteq \mathcal{S}$ (where $|\mathcal{L}| \ll |\mathcal{S}|$ because our ``cost'' scales with $|\mathcal{L}|$) and train $\bar{h} = \arg\max_{h \in \mathcal{H}} \hat{\epsilon}_{\mathcal{L}}(h)$ with the goal of minimizing $\epsilon_{\mathcal{D}}(\bar{h})$. We are given full access to $\mathcal{S}$ until we make our purchase decision, at which point we can use \textit{only} $\mathcal{L}$ to choose our final hypothesis (i.e., we must ``forget'' everything we know about $\mathcal{S} \setminus \mathcal{L}$). This can be thought of as choosing the smallest possible representative subsample $\mathcal{L}$. Intuitively, it is similar to a set cover problem: we want to pose queries that ``cover'' (i.e., eliminate) as many false hypotheses (inconsistent with our labeled data set) as possible. This problem clearly has submodular structure.

\begin{lemma}\label{lma:ss}
The \textbf{selective sampling on a budget} problem is submodular and monotone decreasing.
\end{lemma}

\begin{proof}
We provide a non-rigorous justification. First, for labeled subsample $\mathcal{A}$, define a hypothesis set $\mathcal{H}_{\mathcal{A}} \subseteq \mathcal{H}$ that contains all hypotheses from $\mathcal{H}$ that are consistent with the labeled points in $\mathcal{A}$: $\mathcal{H}_{\mathcal{A}} = \{h : \hat{\epsilon}_{\mathcal{A}}(h)=0 \wedge h \in \mathcal{H}\}$. Now define a function $f(\mathcal{A}) = 1-|\mathcal{H}_{\mathcal{A}}| / |\mathcal{H}|$, i.e., maps $\mathcal{A}$ to the value of 1 minus probability mass (under a uniform prior) of its consistent hypothesis set. Now consider labeled subsamples $\mathcal{B}$ and $\mathcal{B}'$ such that $\mathcal{B} \subseteq \mathcal{B}' \subseteq \mathcal{S}$ and arbitrary point $X \not\in \mathcal{B}, \not\in \mathcal{B}'$. The key insight here is that as we add labeled points to our subsamples, we can only remove hypotheses from our current hypothesis space; once a hypothesis has been removed, it cannot be re-added.

\textbf{$f$ is monotone increasing:} Suppose that hypothesis $h \in \mathcal{H}_{\mathcal{B}'}$. This means that it is consistent with every labeled point in $\mathcal{B}'$. Because $\mathcal{B} \subseteq \mathcal{B}'$, $h$ must also be consistent with every point in $\mathcal{B}$ and so $h \in \mathcal{H}_{\mathcal{B}}$. Therefore, $\mathcal{H}_{\mathcal{B}'} \subseteq \mathcal{H}_{\mathcal{B}}$ and so
\begin{eqnarray*}
|\mathcal{H}_{\mathcal{B}'}| &\leq& |\mathcal{H}_{\mathcal{B}}| \\
|\mathcal{H}_{\mathcal{B}'}| / |\mathcal{H}| &\leq& |\mathcal{H}_{\mathcal{B}}| / |\mathcal{H}| \\
1-|\mathcal{H}_{\mathcal{B}'}| / |\mathcal{H}| &\geq& 1-|\mathcal{H}_{\mathcal{B}}| / |\mathcal{H}| \\
f(\mathcal{B}') &\geq& f(\mathcal{B})
\end{eqnarray*}

\noindent whenever $\mathcal{B} \subseteq \mathcal{B}'$.

\textbf{$f$ is submodular:} Now suppose that adding point $X$ to $\mathcal{B}'$ removes $h$ from $\mathcal{H}_{\mathcal{B}'}$, i.e., $h \in \mathcal{H}_{\mathcal{B}'} \setminus \mathcal{H}_{\mathcal{B}' \cup \{X\}}$. Because $\mathcal{H}_{\mathcal{B}'} \subseteq \mathcal{H}_{\mathcal{B}}$ whenever $\mathcal{B} \subseteq \mathcal{B}'$, it must also be the case that $\mathcal{H}_{\mathcal{B}' \cup \{X\}} \subseteq \mathcal{H}_{\mathcal{B} \cup \{X\}}$ and so $h \in \mathcal{H}_{\mathcal{B}} \setminus \mathcal{H}_{\mathcal{B} \cup \{X\}}$. Thus, if adding $X$ to $\mathcal{B}'$ removes $m$ hypotheses from $\mathcal{H}_{\mathcal{B}'}$, then adding $X$ to $\mathcal{B}$ must remove $n \geq m$ hypotheses from $\mathcal{H}_{\mathcal{B}}$:
\begin{eqnarray*}
m &\leq& n \\
|\mathcal{H}_{\mathcal{B}'}| - |\mathcal{H}_{\mathcal{B}'}| + m &\leq& |\mathcal{H}_{\mathcal{B}}| - |\mathcal{H}_{\mathcal{B}}| + n \\
|\mathcal{H}_{\mathcal{B}'}|/|\mathcal{H}| - (|\mathcal{H}_{\mathcal{B}'}| - m)/|\mathcal{H}| &\leq& |\mathcal{H}_{\mathcal{B}}|/|\mathcal{H}| - (|\mathcal{H}_{\mathcal{B}}| - n)/|\mathcal{H}| \\
-1 + |\mathcal{H}_{\mathcal{B}'}|/|\mathcal{H}| + 1 - (|\mathcal{H}_{\mathcal{B}'}| - m)/|\mathcal{H}| &\leq& -1 + |\mathcal{H}_{\mathcal{B}}|/|\mathcal{H}| + 1 - (|\mathcal{H}_{\mathcal{B}}| - n)/|\mathcal{H}| \\
1 - (|\mathcal{H}_{\mathcal{B}'}| - m)/|\mathcal{H}| - (1 - |\mathcal{H}_{\mathcal{B}'}|/|\mathcal{H}|) &\leq& 1 - (|\mathcal{H}_{\mathcal{B}}| - n)/|\mathcal{H}| - (1 - |\mathcal{H}_{\mathcal{B}}|/|\mathcal{H}|) \\
f(\mathcal{B}' \cup \{X\}) - f(\mathcal{B}') &\leq& f(\mathcal{B} \cup \{X\}) - f(\mathcal{B})
\end{eqnarray*}

\noindent whenever $\mathcal{B} \subseteq \mathcal{B}'$.
\end{proof}
\noindent This is intuitive. If $\mathcal{H}' \subseteq \mathcal{H}$ contains all hypotheses in $\mathcal{H}$ that are inconsistent with $X$'s label, then clearly $\mathcal{H}_{\mathcal{B}'} \subseteq \mathcal{H}_{\mathcal{B}}$ implies that $\mathcal{H}_{\mathcal{B}'} \setminus \mathcal{C} \subseteq \mathcal{H}_{\mathcal{B}} \setminus \mathcal{C}$.

This result may not seem terribly exciting, but what it does suggest is that we can solve the selective sampling on a budget problem using a greedy approach: on the $t$th iteration, choose the $X$ that eliminates the largest number of inconsistent hypotheses from our current $\mathcal{H}_{\mathcal{B}}$:
\[
(X, Y)_t = \underset{(X,Y) \in \mathcal{S} \setminus \mathcal{B}_t}{\arg\max} \left|\sum_{h \in \mathcal{H}_{\mathcal{B}_t}} \mathds{1}\{h(X) \not= Y\} \right|
\]

\subsubsection{Greedy active learning is adaptive submodular}

Now imagine a variation of the above selective sampling problem where we do not have access to the label of $X \in \mathcal{S}$ until we ``purchase'' it. Here we might use \textbf{active learning}. Active learning is a variation of the supervised learning paradigm where the learner does not receive access to a fully labeled data sample $\mathcal{S}$ upfront. Rather it has access to an unlabeled data sample $\mathcal{U} = \{(X, ?)\}$, as well as an \textit{oracle} that the learner can \textit{query} for the response (or label) of an observation, $Y = \mathbf{or}(X)$ \cite{dasgupta2011}. The active learner is given agency to choose which individual samples to label, but each query has a cost $c$ and the learner has only a limited \textit{budget} to spend on labeling data. Similar to the selective sampling scenario described above, the active learner has dual goals: to choose simultaneously a labeled subset of observations $\mathcal{L} \subseteq \mathcal{U}$ and a hypothesis $\bar{h} = \arg\min_{h \in \mathcal{H}} \hat{\epsilon}_{\mathcal{L}}(h)$ (i.e., $\bar{h}$ is the ERM for $\mathcal{L}$) that will yield the best possible predictive performance (i.e., lowest risk $\epsilon_{\mathcal{D}}(\bar{h}))$.

When evaluating active learning algorithms, we are concerned primarily with two performance properties: the quality (in terms of risk) of the hypotheses they produce and their query efficiency. Intuitively, a good active learner will use a very small number of label queries to produce a hypothesis with very small predictive error. More formally, we are interested in (1) how the error of the hypothesis produced by an active learner that chooses labeled subsample $\mathcal{L}$ compares with that of the hypothesis that we could learn from a fully labeled sample $\mathcal{S}$ where $\mathcal{L} \subseteq \mathcal{S}$; and (2) how many label queries must be made to achieve a certain level of performance, which we call \textbf{label complexity} and express in \textit{Big-Oh} notation. An ideal active learner will compete with fully supervised learning with $|\mathcal{L}| \ll |\mathcal{S}|$. More realistically, we hope to at least place an upper bound on the error of active learning that is within a constant (multiplicative or additive) factor of the error of fully labeled supervised learning.

It is not hard to design greedy approaches to active learning, but two questions arise: first, are there greedy active learning algorithms that have sound theoretical guarantees about error and label complexity; and second, can we show that such algorithms are in fact specific cases of more general approaches based on submodularity? The answer to both of these questions is, in fact, yes \cite{dasgupta2004}. Let $\mathcal{L}_t$ be the set of labeled data points after $t$ queries (and recall that $\mathcal{H}_{\mathcal{L}_t}$ is the set of hypotheses from $\mathcal{H}$ consistent with the labeled data in $\mathcal{L}_t$). For the next ($t+1$) label query, we want to choose the unlabeled point that provokes the greatest disagreement between hypotheses in $\mathcal{H}_{\mathcal{L}_t}$.The maximum disagreement occurs when half of the hypotheses predict one label and the rest the other. Equivalently, this minimizes the absolute value of the sum of all predicted labels: $| \sum_{h \in \mathcal{H}_{\mathcal{L}_t}} h(x)|$ (when using $\pm1$ labels). Following this query policy, we hope to cut the $t$th hypothesis space roughly in half with query $t+1$ and achieve a label complexity that is roughly $O(\log (d/\varepsilon))$ for $d$ the size (e.g., VC dimension) of the hypothesis space and deviation bound $\varepsilon$. \cite{dasgupta2004} shows that in the worst case, this strategy may have to query every single label; indeed, for certain pathological cases, even the optimal query strategy will need to query every label. However, the average case analysis is much more promising. On average, the greedy strategy's label complexity is at most $\widetilde{O}(\log d)$ times larger that that of the optimal policy, as we show below in \textbf{Theorem \ref{thm:dasgupta}}, which rephrases \textit{Claim 4} and \textit{Theorem 3} from \cite{dasgupta2004}:\\

\begin{theorem}[Dasgupta ~\cite{dasgupta2004}]
\label{thm:dasgupta}
Suppose the optimal query policy requires $M$ labels in expectation for target hypotheses chosen uniformly from hypothesis class $\mathcal{H}$ of (VC) dimension $d \geq e^e \approx 16$. Then the expected number of labels queried by the greedy strategy is at least $\frac{M \log d}{\log \log d}$ and at most $4 M \log d$.
\end{theorem}

As \cite{dasgupta2004} points out, the lower bound is a bit depressing, but we derive some comfort from the fact that the upper bound matches the lower bound within a multiplicative factor. We can extend this analysis to a Bayesian framework where we have a nonuniform prior distribution over hypotheses $\pi(h)$. In this setting, we seek a label query that will divide the \textit{probability mass} over hypotheses (rather than the hypothesis space itself) in half. We do this by minimizing the absolute value of the sum of predictions weighted by the prior probabilities of the hypotheses making them: $| \sum_{h \in \mathcal{H}_t} \pi(x) h(x)|$. In this case, the $d$ term in the lower and upper bounds is replaced with $\min_{h \in \mathcal{H}} \pi(h)$.

In \cite{Golovin}, the authors show that the hypothesis space reduction problem is adaptive submodular, specifically an example of an adaptive stochastic coverage problem. Here our ground set is the set of all points $V = \{x : x \in \mathcal{U}\}$, and each point has an unobserved state $O = \{y : y \in \mathcal{Y}\} = \{\pm1\}$ where $\Phi(x) = y$ for the pair $(x,y)$ and for a given hypothesis $h \in \mathcal{H}$, $\Phi_h(x) = h(x)$. The set of labeled points $\mathcal{L}_t$ forms a consistent partial realization $\chi_t$ at iteration $t$. Then we can define a function that takes as input an element subset $V'$ and a realization function $\Phi'$ and maps it to a real number in the interval $[0,1]$:
\[
f(V', \Phi') = \hat{f}(H = \{h : h(x) = \Phi'(x) \mbox{ for all } x \in V'\}) = 1 - \sum_{h \in H} \pi(h)
\]

\noindent So for a labeled subset $\mathcal{L}_t$, $f(V_t, \Phi_t) = \hat{f}(\mathcal{L}_t)$, where $V_t = \{ x : (x,y) \in \mathcal{L}_t \}$ and $\Phi_t(x) = \Psi_t(x) = y$ for $(x, y) \in \mathcal{L}_t$. This function is adaptively submodular, as shown in \textbf{Lemma \ref{lma:gbs}}, adapted from \cite{Golovin}:\\

\begin{lemma}[Golovin and Krause ~\cite{Golovin}]
\label{lma:gbs}
The \textbf{hypothesis space reduction} problem is adaptive submodular and adaptive monotone.
\end{lemma}

\noindent The proof of monotonicity follows along lines similar to the one used in \textbf{Theorem \ref{lma:ss}}: basically, querying a label can only remove hypotheses, and hypothesis probabilities are nonzero, so removing one can only reduce the value of $f$. The proof of submodularity is more subtle, though it rests on the same intuition as that of monotonicity and involves comparing the conditional expected marginal benefits, as described in. Interestingly, this angle yields a slightly more optimistic average case analysis than that given in \textbf{Theorem \ref{thm:dasgupta}} above, removing the constant multiplier from the upper bound. We give it below in \textbf{Theorem \ref{thm:golovin}}, adapted from \cite{Golovin}:\\

\begin{theorem}[Golovin and Krause ~\cite{Golovin}]
\label{thm:golovin}
Suppose the optimal query policy requires $M$ labels in expectation for target hypotheses chosen using distribution $\pi$ from hypothesis class $\mathcal{H}$. Then the expected number of labels queried by the greedy strategy is at most $M \left(\log \left(\frac{1}{\min_{h \in \mathcal{H}} \pi(h)}\right) + 1\right)$.
\end{theorem}

\subsubsection{New directions}

\noindent This is a wonderful example of cross fertilization between research in computer science theory and optimization and learning theory. Working on submodularity and adaptive submodularity, computer scientists were able to rediscover and generalize previously published results from machine learning, improving an upper bound along the way. More important, they provided new and useful insights into the problem, relating it to other problems (which we did not discuss in this section) and paving the way to new discoveries. Recently, there has been an explosion of similar work, much of it published in 2013. \cite{mirzasoleiman13distributed} describe a framework for performing distributed submodular maximization in a shared-nothing (MapReduce) storage setting and using it to choose a representative subsample of a massive data set for learning (similar to our selective sampling on a budget problem). \cite{chen13near} describe a greedy \textit{batch-mode} active learning algorithm that queries labels in batches of size $k > 1$ and show that this approach is competitive not only with optimal batch-more active learning but also with more traditional greedy active learning. There are a variety of other papers pushing the boundary in this area \cite{GolovinK11} \cite{golovin10near} \cite{zuluaga13active} \cite{hollinger:2011}.

There are two lines of work that seem conspicuously absent (at least, based on our admittedly myopic literature review): (1) applications of submodularity to \textbf{streaming active learning}; and (2) ``aggressive'' active learning in the nonrealizable case. The former involves active learning when we do \textit{not} have access to the entirety of $\mathcal{U}$ at the start of the learning process. Rather, we receive one data point at a time in an online fashion and must make a query decision for point $X_t$ based only on the samples $\mathcal{U}_t$ that we've seen so far. The above greedy algorithm and analysis require that we be able to choose a point $X_t = \arg\min_{X \in \mathcal{U} \setminus \mathcal{L}_t} \left| \sum_{h \in \mathcal{H}_{\mathcal{L}_t}} \pi(h) h(X) \right|$. We cannot, of course, do this in the streaming setting. Nonetheless, intuition suggests that we may be able to extend the adaptive submodularity framework (or some related idea) to this setting.

Aggressive active learning in the non-realizable case is a wide open problem, at least as of \cite{dasgupta2011}. Informally, aggressive active learners, which include greedy active learners, are those that attempt to make the ``most informative'' label query at each step. In the realizable case, it is possible to develop aggressive algorithms that are statistically consistent (will discover the optimal hypothesis with enough queries) and have sound theoretical guarantees for label complexity. However, these guarantees go out the window with realizability. Perhaps some form of submodularity may help here, but at first blush, it looks as though the nonrealizable case will not satisfy the assumptions necessary for adaptive submodularity. We \textit{do} have a variety of \textit{mellow} active learners, which seek \textit{any} informative label query, that are label efficient and statistically consistent \cite{iwalcal} \cite{dasgupta2011}. It would be interesting to develop a new analysis of these algorithms in terms of submodularity and then to see if this analysis perhaps provides a bridge between mellow and aggressive active learning.

\section{Submodularity in Weighted Constraint Reasoning}

  Many application require efficient representation and reasoning about factors like fuzziness, probabilities, preferences, and/or costs. Various extensions to the basic framework of Constraint Satisfaction Problems (CSPs) \cite{D:BOOK:03} have been introduced to incorporate and reason about such ``soft'' constraints. These include variants like \emph{fuzzy-CSPs}, \emph{probabilistic-CSPs}, and Weighted-CSPs (WCSPs). A WCSP is an \emph{optimization} version of a CSP in which the constraints are no longer ``hard,'' but are extended by associating (non-negative) \emph{costs} with the tuples. The goal is to find an assignment of values to all variables from their respective domains such that the total cost is \emph{minimized}.
  
  For simplicity, we restrict ourselves to Boolean WCSPs. Note that this class can be used to model important combinatorial problems such as representing and reasoning about user preferences \cite{BBDHP:JAIR:04}, over-subscription planning with goal preferences \cite{D:IJCAI:07}, combinatorial auctions \cite{S:AI:02}, and bioinformatics \cite{SGSST:AI:07}, energy minimization problems in probabilistic settings, computer vision, Markov Random Fields \cite{K:MSR:05}, etc. In addition, many real-world domains exhibit \emph{submodularity} in the cost structure that is worth exploiting for computational benefits. In what follows, we define a class of submodular constraints over Boolean domains and give a polynomial-time algorithm for solving instances from this class. 
  
\subsection{Weighted Constraint Satisfaction Problems}
  Formally, a WCSP is defined by a triplet $\langle \mathcal{X,D,C} \rangle$ where $\mathcal{X} = \{ X_1,X_2 \ldots X_N \}$ is a set of \emph{variables}, and $\mathcal{C} = \{ C_1,C_2 \ldots C_M \}$ is a set of \emph{weighted constraints} on subsets of the variables. Each variable $X_i$ is associated with a discrete-valued \emph{domain} $D_i \in \mathcal{D}$, and each constraint $C_i$ is defined on a certain subset $S_i \subseteq \mathcal{X}$ of the variables. $S_i$ is referred to as the \emph{scope} of $C_i$; and $C_i$ specifies a non-negative \emph{cost} for every possible combination of values to the variables in $S_i$. The \emph{arity} of the constraint $C_i$ is equal to $|S_i|$. An \emph{optimal} solution is an assignment of values to all variables (from their respective domains) so that the \emph{sum} of the costs (as specified locally by each weighted constraint) is \emph{minimized}. In a Boolean WCSP, the size of any variable's domain is $2$ (that is, $D_i = \{ 0,1 \}$ for all $i$). Boolean WCSPs are representationally as powerful as WCSPs; and it is well known that optimally solving Boolean WCSPs is NP-hard in general \cite{D:BOOK:03}. The \emph{constraint graph} associated with a WCSP instance is an undirected graph where a node represents a variable and an edge $(X_i,X_j)$ exists if and only if $X_i$ and $X_j$ appear together in some constraint.

\subsection{Submodular Constraints}
  Submodular constraints over Boolean domains correspond directly to \emph{submodular set functions}. A set function $\psi : 2^V \rightarrow \mathrm{Q}$ defined on all subsets of a set $V$ is \emph{submodular} if and only if, for all subsets $S, T \subseteq V$, we have $\psi(S \cup T) + \psi(S \cap T) \leq \psi(S) + \psi(T)$. A submodular constraint is a weighted constraint with a submodular cost function. Here, the correspondence is in light of the observation that any subset $S$ can be interpreted as specifying the Boolean variables in $V$ that are set to $1$. Boolean WCSPs with submodular constraints are known to be tractable \cite{ZSH:CSC:10}. However, the general algorithm for solving Boolean WCSPs with submodular constraints has a time complexity of $O(N^6)$, which is not very practical. Specific classes of submodular constraints have been shown to be related to graph cuts, and are therefore solvable more efficiently \cite{ZSH:CSC:10}.

\subsection{Lifted Graphical Representations for Weighted Constraints}
  \emph{Constraint Composite Graphs} (CCGs) are combinatorial structures associated with optimization problems posed as WCSPs. They provide a unifying framework for exploiting both the \emph{graphical} structure of the variable interactions as well as the \emph{numerical} structure of the weighted constraints \cite{K:CP:08}. We reformulate WCSPs as \emph{minimum weighted vertex cover} problems to construct simple bipartite graph representations for important classes of submodular constraints, thereby translating them into max-flow problems on bipartite graphs.

  The concept of the minimum weighted VC\footnote{A \emph{vertex cover} (VC) is a set of nodes $S$ such that every edge has at least one end point in $S$.} on a given undirected graph $G=\langle V,E \rangle$ can be \emph{extended} to the notion of projecting minimum weighted VCs onto a given IS\footnote{$U$ is an \emph{independent set} (IS) of a graph if and only if no two nodes in $U$ are connected by an edge.} $U \subseteq V$. The input to such a projection is the graph $G$ as well as an identified IS $U=\{u_1, u_2 \ldots u_k\}$. The output is a \emph{table} of $2^k$ numbers. Each entry in this table corresponds to a $k$-bit vector. We say that a $k$-bit vector imposes the following restrictions: if the $i^{th}$ bit is $0$ ($1$), the node $u_i$ is necessarily excluded (included) from the minimum weighted VC. The value of an entry is the weight of the minimum weighted VC \emph{conditioned} on the restrictions imposed by it. Figure \ref{VC_example} presents a simple example.
  
\begin{figure}[t]
  \centering
  \includegraphics[width=0.35\textwidth]{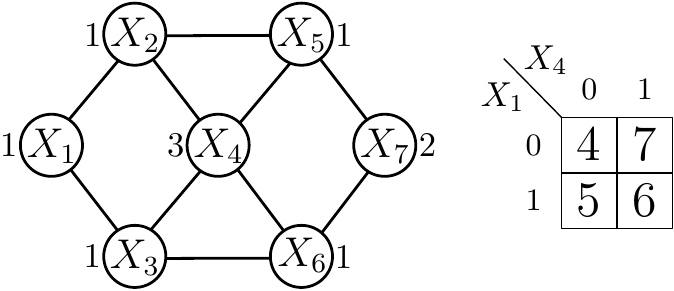}
  \caption{The table on the right-hand side represents the projection of the minimum weighted VC problem onto the IS $\{ X_1, X_4 \}$ of the node-weighted undirected graph on the left-hand side. (The weights on $X_4$ and $X_7$ are set to $3$ and $2$, respectively, while all other nodes have unit weights.) The entry `$7$' in the cell $(X_1=0, X_4=1)$, for example, indicates that, when $X_1$ is prohibited from being in the minimum weighted VC but $X_4$ is necessarily included in it, then the weight of the minimum weighted VC - $\{ X_2, X_3, X_4, X_7 \}$ or $\{ X_2, X_3, X_4, X_5, X_6 \}$ - is $7$.}
  \label{VC_example}
\end{figure}

  The aformentioned table can be viewed as a weighted constraint over $|U|$ Boolean variables. Conversely, given a (Boolean) weighted constraint, we can think about designing a ``lifted'' representation for it so as to be able to view it as the projection of a minimum weighted VC problem in some node-weighted undirected graph. This idea was first discussed in \cite{K:ISAIM:08}. The benefit of constructing these graphical representations for individual constraints lies in the fact that the ``lifted'' graphical representation for the entire WCSP can be obtained simply by ``merging\footnote{nodes that represent the same variable are simply ``merged'' - along with their edges - and every ``composite'' node is given a weight equal to the sum of the individual weights of the merged nodes.}'' them. This ``merged'' graph is referred to as the CCG associated with the WCSP. Computing the minimum weighted VC for the CCG yields a solution for the WCSP; namely, if $X_i$ is in the minimum weighted VC, then it is assigned the value $1$ in the WCSP, else it is assigned the value $0$ in the WCSP. Figure \ref{VC_example} shows an example WCSP and its CCG.
  
\begin{figure}[t]
  \centering
  \includegraphics[width=0.9\textwidth]{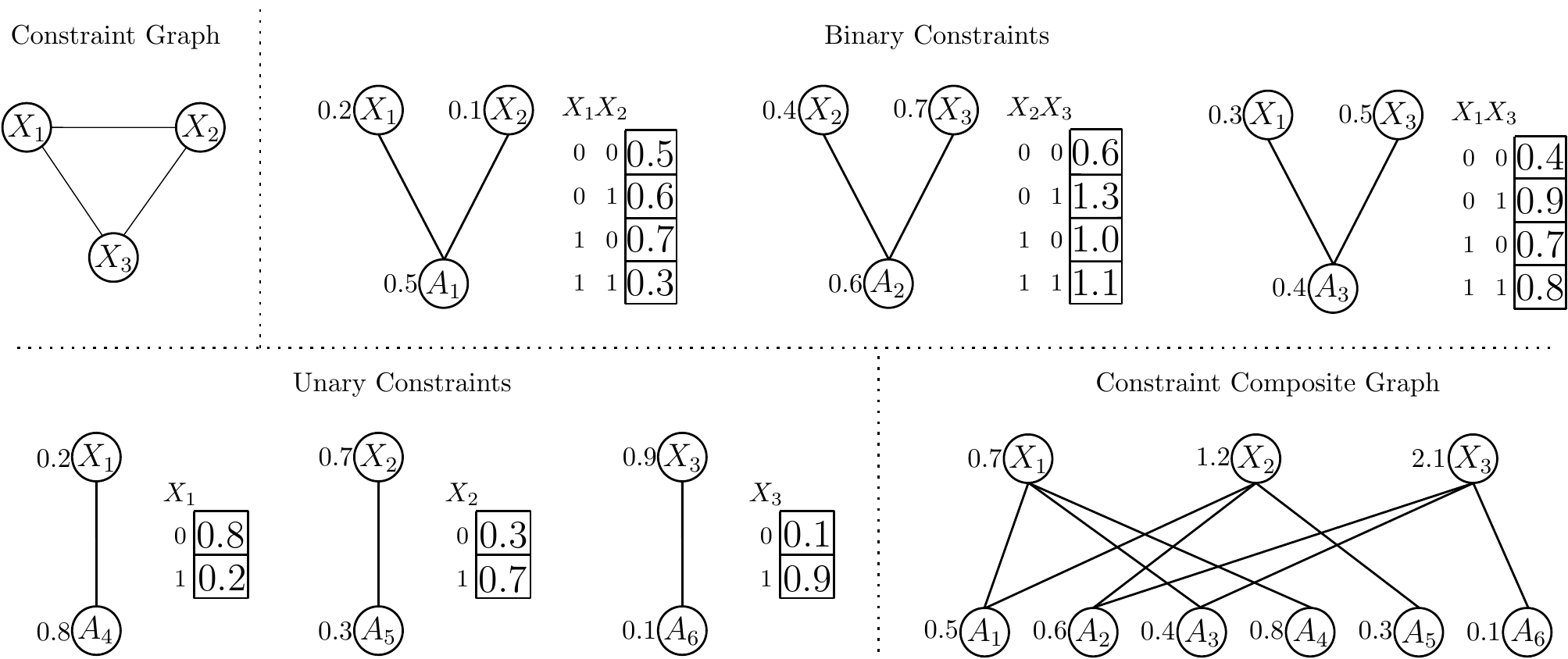}
  \caption{Shows a WCSP over $3$ Boolean variables. The constraint network is shown in the top-left cell, and the $6$ binary and unary weighted constraints are shown along with their lifted graphical representations in the $1^{st}$ and $2^{nd}$ rows. The CCG is shown in the bottom-right cell.}
  \label{ccg_example}
\end{figure}

  Any given weighted constraint on Boolean variables can be represented graphically using a \emph{tripartite} graph, which can be constructed in polynomial time \cite{K:CP:08}. In many cases, the lifted graphical representations even turn out to be only \emph{bipartite}. Since the resulting CCG is also bipartite if each of the individual graphical representations are bipartite, the tractability of the language $\mathcal{L}_{bipartite}^{Boolean}$ - the language of all Boolean weighted constraints with a bipartite graphical representation - is readily established. This is because solving minimum weighted VC problems on bipartite graphs is reducible to max-flow problems, and can therefore be solved efficiently in polynomial time.
  
  Finally, Boolean weighted constraints can be represented as multivariate polynomials on the variables participating in that constraint \cite{ZSH:CSC:10,K:CP:08}. The coefficients of the polynomial can be computed with a standard Gaussian Elimination procedure for solving systems of linear equations. The linear equations themselves arise from substituting different combinations of values to the variables, and equating them to the corresponding entries in the weighted constraint. One way to build the CCG of a given weighted constraint is: (a) build the graphical representations for each of the individual terms in the multivariate polynomial; and (b) ``merge'' these graphical representations \cite{K:CP:08}.

\subsection{Submodular Constraints with bounded arity}
  The focus of \cite{KCK:SARA:13} is on bounded arity submodular constraints (that is, submodular constraints with arity at most $K$, for some constant $K$) and providing asymptotically improved algorithms for solving them. The reason these submodular constraints can be solved more efficiently is because the underlying max-flow problems are staged on bipartite graphs. For Boolean WCSPs with arity at most $K$, the bipartite CCG has $N$ nodes in one partition, at most $2^K M$ nodes in the other partition, and at most $K 2^K M$ edges. For $K$ bounded by a constant, this results in a time complexity of $O(N M \log M)$. This significantly improves on the $O((N+M)^{3})$ time complexity of the algorithm provided by \cite{ZSH:CSC:10}.\footnote{For arity $K$, $M$ could be as large as $N \choose K$.} Figure \ref{fig_lifted_rep} shows the lifted graphical representation for all possible terms of a constraint of arity $3$.
  
\begin{figure}[t]
\includegraphics[width=\textwidth]{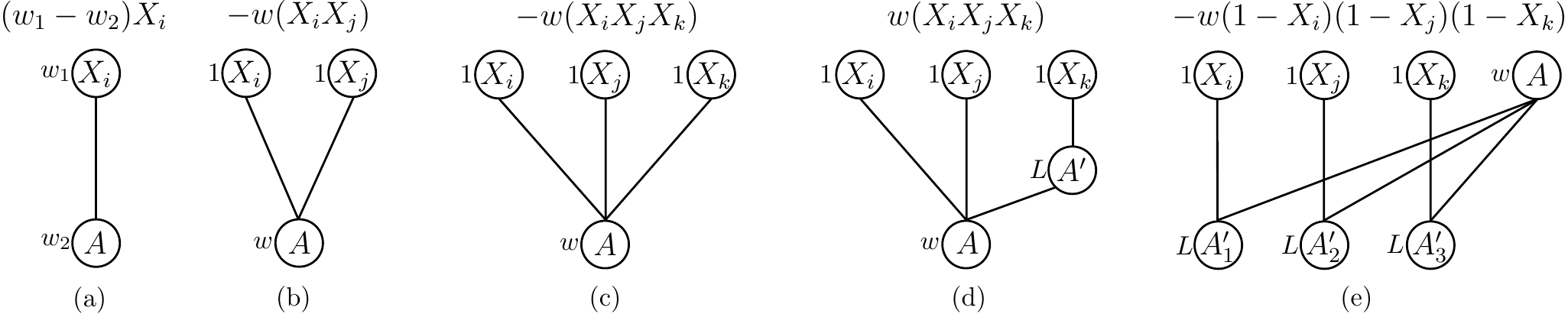}
\caption{Lifted graphical representations for different kind of terms. (a) represents a linear term, either positive or negative, where $w_1$ and $w_2$ are chosen appropriately. (b) represents a negative quadratic term. (c) represents a negative cubic term. (d) illustrates the ``change of variable'' method and essentially represents a leading positive cubic term. (e) is the same ``flower'' structure shown in (c), but with a ``thorn'' introduced for each variable. The resulting graph is also bipartite, because the auxiliary variable $A$ can be moved to the same partition as the original variables. The graph now represents the term $-w (1-X_i)(1-X_j)(1-X_k)$, which in effect, is a bipartite representation for an expression with a leading positive cubic term.}
\label{fig_lifted_rep}
\end{figure}

\subsection{Social Influence}
With the increasing popularity of online social network websites and apps, such as Facebook and Twitter, social networks now play a fundamental role as medium for people to share, exchange, and obtain new ideas and information. Modeling, utilizing, and understanding \textbf{social influence} has become ``hot topic'' in computer science, machine learning, and computational social science \cite{GJA10,KKT03,GB12,Bharathi:2007,Borodin:2010,chen2011influence,he2012influence,Budak:2011}. It turns out that submodular functions and especially submodular function maximization play fundamental roles in solving algorithmic questions associated with social influence. In this section, we mainly focus on using submodular function maximization techniques to solve two problems related to social influence, namely \textit{influence maximization}~\cite{KKT03} and \textit{network inference}~\cite{GJA10}.

\subsubsection{Influence Maximization}
Assume that now a company wants to promote its new product among the individuals in a social network (real or virtual). The company has a limited budget to give free samples of their new product to the users in the social network. A natural question to ask is to which set of users should the company give the free sample products, such that the overall adoption of the new product can be maximized. The question is exactly the \textbf{influence maximization} problem, namely selecting a small set of seed nodes in a social network such that its overall influence coverage is maximized under a certain diffusion model.

Among the many diffusion models, the \textbf{Independent Cascade} (IC) model and \textbf{Linear Threshold} (LT) are used widely in the study of influence maximization~\cite{KKT03}. Both IC and LT models are stochastic models characterizing how influence propagates throughout the network, starting from the initial seed notes.

%

For influence maximization, the objective function $\sigma(S)$, where $S$ is the initial seed set, is the expected number of activated nodes under the diffusion model. The problem is simply to maximize $\sigma(S)$ subject to the cardinality constraint $|S|\leq k$.

It has been shown that the influence maximization problem under both IC model and LT model is NP-hard~\cite{KKT03}. However, by the following theorems the problem allows efficient approximation algorithm.
 \begin{theorem}[Kempe et al.~\cite{KKT03}]\label{Thm:KKT}
The objective function of influence maximization problem under both IC and LT model is non-negative, monotone and submodular.
\end{theorem}
By the classic greedy algorithm for monotone submodular function maximization, a $1-1/e$ approximation guarantee can be achieved. The proof of the theorem is by the fact that conic combination of submodular functions is also submodular. The objective function is a expectation and can be written as
$$
\sigma(S) = \sum_{\text{outcome }X}Prob[X]\sigma(S|X),
$$
where $X$ is any realization of the stochastic diffusion process. A reachability argument for both IC model and LT model can be used to show that $\sigma(S|X)$ is submodular under any $X$. 

Though the greedy algorithm can solve the problem approximately in polynomial time, its key step, i.e., evaluation of the marginal gain $\sigma(S\cup\{u\})-\sigma(S)$, can take a long time for large networks. Many papers have been published on how to improve the efficiency of the algorithm~\cite{Chen:2010,Leskovec:2007:COD:1281192.1281239,Goyal:2011:COG:1963192.1963217,DBLP:conf/icdm/ChenYZ10,Chen:2010:SIM:1835804.1835934}, such as by lazy evaluation~\cite{Leskovec:2007:COD:1281192.1281239,Goyal:2011:COG:1963192.1963217}, or by approximate evaluation of the marginal gain~\cite{Chen:2010,DBLP:conf/icdm/ChenYZ10,Chen:2010:SIM:1835804.1835934} .

A more general result on \textbf{Generalized Linear Threshold} models has been proved generalizing the results for the IC and LT models in~\cite{Mossel2007}. The proof uses a sophisticated stage-wise coupling argument to show that submodularity applies. The idea of the proof is to add the initial seeds and propagate the influence stage by stage. The key component in the proof is the anti-sense coupling used in the last stage.

A extension to influence maximization that draws much attention recently is to solve this problem under the competitive influence. Competitive influence implies that two or multiple competitive products, or ideas are propagating simultaneously in the social network. The influence maximization problem naturally extends to maximizing one's own influence~\cite{Bharathi:2007,Borodin:2010,chen2011influence} or
minimizing the influence of the competitors~\cite{he2012influence,Budak:2011} given the choices of the initial seeds of the competitors. For example, on the maximization side,~\cite{ChenCCKLRSWWY11} study the influence maximization when a user can dislike the product and propagate bad news about it. On the minimization side, ~\cite{HSCJ2012} study the idea of \textit{influence blocking maximization}, which focuses on selecting seeds to block the propagation of rumors. Both approaches solve the optimization problem by showing the objective function is monotone and submodular. The proof technique is similar to that in~\cite{KKT03}. However, their arguments are much more complicated due to the interaction of competitive diffusion.

\subsubsection{Network Inference}
The influence maximization problem takes as its input a social network structure with the strength of influence on each edge. However, in most cases, the underlying network that enables the diffusion is hidden (e.g., networks on who influenced whom). The most common observations of information diffusion are only the activation time for individual in social network (e.g., the time stamps of when a person posted a blog containing certain information or retweeted another user's tweet, when a user bought a certain product in viral marketing applications, etc.). The \textbf{network inference problem} focuses on discovering the diffusion network from the observed cascades occurring among the individuals in a social network. Existing approaches to this problem solve a maximum likelihood estimation problem with respect to the network structure under certain diffusion models~\cite{GJA10,GB12,GDB11,SJ10}. It turns out that the likelihood function of this problem can be approximated with a submodular function. Thus, submodular function maximization can be used to solve this problem~\cite{GJA10,GB12}.

The extended IC model is used as the diffusion model in~\cite{GJA10,GB12}. In the extended IC model, each edge is associated with an activation probability. Moreover, each activation has time delay. For example, if node $v$ is activated at time $t_v$ and the activation attempt to $v$'s neighbour $u$ succeeds. Then $u$ with become activated at time $t_u+\Delta t$, where $\Delta t$ is the time delay. In the model, the delay time satisfies exponential distribution or power law distribution, namely
$$
P_d(\Delta t)\propto e^{-\frac{\Delta t}{\alpha}} \text{\ or\ } P_d(\Delta t)\propto\frac{1}{\Delta t^\alpha}
$$

Then according to the model, if $v$ which is activated at time $t_v$ and $v$ succeeds in activating node $u$ which becomes activated at time $t_u$. Then the probability this activation occurs is:
$$
P_c(v,u)=P_d(t_u-t_v)p_{vu}
$$
Also the model assumes that influence can only propagate forward in time, which means $
P_c(v,u)=0 \text{\ if\ }\ t_v>t_u.
$
Then if the pattern of cascade $c$ forms a tree $T$, the probability that the cascade is observed given the tree is
$$
P(c|T)=\prod_{(i,j)\in T}P_c(i,j)
$$
In addition, if we assume the who-infect-who relation forms a tree pattern (one is only activated by one person). Then given a certain $G$, the probability we observe the cascade $c$ would be
$$
P(c|G)=\sum_{T\in T(G)}P(c|T)P(T|G)\propto \sum_{T\in T(G)}\prod_{(i,j)\in T}P_c(i,j)
$$
where $T(G)$ is all directed spanning tree on $G$.
Therefore, if we have observed a set of cascades $C=\{c_1,c_2,\ldots\}$,
The probability of observing all these cascades is
$$
P(C|G)=\prod_{c\in C}P(c|G).
$$
Under this configuration, the network inference problem is to find an graph $\hat{G}=(V,\hat{E})$ with less than $k$ edges such that
$$
\hat{G} = \arg\max_{|E|\leq k} P(C|G)
$$
In the objective function, we sum over all spanning trees of the graph $G$, which is super-exponential. In order to make this computation feasible, we instead solve an approximation in which only the spanning tree with the maximal likelihood is considered, namely
$$
P(C|G)=\prod_{c\in C}\max_{T\in T(G)}P(c|T)=\prod_{c\in C}\max_{T\in T(G)}\prod_{(i,j)\in T}P_c(i,j)
$$
Then we define $F_c(G)$ as the difference between the log likelihood of cascade $c$ over graph $G$ and empty graph $\bar{K}$.
$$
F_c(G)=\max_{T\in T(G)}\log P(c|T)-\max_{T \in T(\bar{K})}\log P(c|T)
$$
and take sum over all the cascades, we have
$$
F_C(G)=\sum_{c\in C}F_c(G)
$$
$\bar{K}$ is a graph with all the nodes in $G$ and also a extra node $m$. The only edges in $\bar{K}$ are the edges from $m$ to every other nodes with activation probability $\varepsilon$ and delay $0$. The extra node represents the external influence. Then the optimization problem can be rewritten as
\begin{equation}\label{Equ:SubObj}
G^*=\arg\max_{|E|\leq k}F_C(G)
\end{equation}

\cite{GJA10} proved that this objective function is monotone and submodular. Therefore, a greedy algorithm can achieve $1-\frac{1}{e}$ approximation for solving it. This approach was later improved by by the \textit{MultiTree algorithm} in ~\cite{GB12}, where the the matrix tree theorem is used to calculated the exact summation over all possible spanning trees, rather than using the maximum spanning tree approximation.

In a on-going project by Xinran He with Prof. Yan Liu, we are experimenting with using \textit{Maximum a Posteriori} inference to solve the network inference problem. The previous approaches assume no prior knowledge about the structure of the inferred graph. However, it has been shown repeatedly that social networks have many unique properties, including heavy-tail degree distribution, small diameter, community structure, and so on. We propose a social network generative model over the diffusion network as a prior to incorporate the prior knowledge about network structure. Our current choice is the Kronecker graphs model~\cite{JDJCZ10,KJ11}. The Kronecker graphs model is a parametric model which can provide a probability for the existence of each edge in the social network. The existence of each edge is considered independent under this model. Using this model as a prior, we can change the objective function $F_C(G)$ in Equation~\ref{Equ:SubObj} to

\[
F'_C(G) = F_C(G) + \sum_{e\in E}(\log Prob[e\ \text{exists}]-\log Prob[e\ \text{not exists}]).
\]

After adding a modular function to a submodular function, the resulted $F'_C(G)$ is still submodular, however it may not necessary be monotone any more. As a result, the simple greedy algorithm can not be used to solve this problem. Instead, we can use the algorithm proposed in~\cite{FeldmanNS11} with a $1/2+o(1)$ approximation guarantee.

\nocite{*}
\bibliographystyle{./IEEEtran}
\bibliography{bibliography}
\end{document}